%% file: emnlp2021.tex
\pdfoutput=1
\documentclass[11pt]{article}

\usepackage{emnlp2021}

\usepackage{times}
\usepackage{latexsym}

\usepackage[T1]{fontenc}
\usepackage[utf8]{inputenc}

\usepackage{microtype}

\input{preamble}

\title{Effects of Parameter Norm Growth During Transformer Training:\\
Inductive Bias from Gradient Descent}

\author{
    William Merrill\footnotemark[1]\;\footnotemark[2] \;
    Vivek Ramanujan\footnotemark[1]\;\footnotemark[3] \;
    Yoav Goldberg\footnotemark[1]\;\footnotemark[4] \;
    Roy Schwartz\footnotemark[5] \;
    Noah A. Smith\footnotemark[1]\;\footnotemark[3] \\
    \footnotemark[1]\; Allen Institute for AI \;
    \footnotemark[2]\; New York University \;
    \footnotemark[3]\; University of Washington \\
    \footnotemark[4]\; Bar Ilan University \;
    \footnotemark[5]\; Hebrew University of Jerusalem \\
    \texttt{willm@nyu.edu} \; \texttt{ramanv@cs.washington.edu}\\
    \texttt{\{yoavg,noah\}@allenai.org} \; \texttt{roys@cs.huji.ac.il}
    }

\begin{document}
\maketitle

\input{sections/abstract}

\input{sections/intro}
\input{sections/observation}
\input{sections/effects}
\input{sections/sat_transformers}
\input{sections/causes}
\input{sections/conclusion}

\section*{Acknowledgments}
We thank Colin Raffel for sharing access to the T5 training checkpoints.
Additional thanks to Qiang Ning, Kyle Richardson, Mitchell Wortsman, Martin Lohmann, and other researchers at the Allen Institute for AI for their comments on the project.

\bibliography{neurips_2020.bib}
\bibliographystyle{acl_natbib}

\appendix
\input{sections/exp_details} % F
\input{sections/uniform}
% A
\input{sections/approx_homogeneity} % C
\input{sections/transformer} % D
\input{sections/integral_appendix}
\input{sections/weight_decay}

\end{document}

%% file: preamble.tex
\usepackage{physics}
\usepackage{bbm}
\usepackage{amsmath}
\usepackage{amsthm}
\usepackage{hyperref}
\usepackage{graphicx}
\usepackage{amssymb}
\usepackage{paralist}

\usepackage{array}

\DeclareMathOperator{\relu}{\mathrm{ReLU}}
\DeclareMathOperator{\softmax}{\mathrm{softmax}}
\DeclareMathOperator{\lnorm}{\mathrm{LN}}
\DeclareMathOperator{\selfattn}{\mathrm{attn}}  % {\mathrm{SA}}
\DeclareMathOperator{\mselfattn}{\mathrm{MSA}}
\DeclareMathOperator{\ff}{\mathrm{FF}}

\DeclareMathOperator{\sgn}{\mathrm{sgn}}

\theoremstyle{definition}
\newtheorem{definition}{Definition}
\theoremstyle{plain}
\newtheorem{theorem}{Proposition}
\newtheorem{lemma}{Lemma}

\newtheorem{assumption}{Assumption}

\def\Snospace~{\S{}}

\usepackage{xpatch}
\makeatletter
\AtBeginDocument{\xpatchcmd{\@thm}{\thm@headpunct{.}}{\thm@headpunct{}}{}{}}
\makeatother

\newcommand{\sat}{\textrm{s}} % s-

\newcommand{\approptoinn}[2]{\mathrel{\vcenter{
  \offinterlineskip\halign{\hfil$##$\cr
    #1\propto\cr\noalign{\kern2pt}#1\sim\cr\noalign{\kern-2pt}}}}}

\newcommand{\appropto}{\mathpalette\approptoinn\relax}

%% file: sections/abstract.tex
\begin{abstract}
The capacity of neural networks like the widely adopted transformer is known to be very high. Evidence is emerging that they learn successfully due to inductive bias in the training routine, typically a variant of gradient descent (GD). To better understand this bias, we study the tendency for transformer parameters to grow in magnitude ($\ell_2$ norm) during training, and its implications for the emergent representations within self attention layers. Empirically, we document norm growth in the training of transformer language models, including T5 during its pretraining. As the parameters grow in magnitude, we prove that the network approximates a discretized network with saturated activation functions. Such ``saturated'' networks are known to have a reduced capacity compared to the full network family that can be described in terms of formal languages and automata. Our results suggest saturation is a new characterization of an inductive bias implicit in GD of particular interest for NLP. We leverage the emergent discrete structure in a saturated transformer to analyze the role of different attention heads, finding that some focus locally on a small number of positions, while other heads compute global averages, allowing counting. We believe understanding the interplay between these two capabilities may shed further light on the structure of computation within large transformers.

\end{abstract}

%% file: sections/intro.tex
\section{Introduction} \label{sec:intro}

% Modern neural networks achieve good generalization performance across many different tasks and application domains.
Transformer-based models \citep{vaswani2017attention} like BERT \citep{devlin2019bert}, XLNet \citep{yang2019xlnet}, RoBERTa \citep{liu2019roberta}, and T5 \citep{raffel2019t5} have pushed the state of the art on an impressive array of NLP tasks.
% Each of these highly overparameterized models is ``pretrained'' on an unsupervised objective resembling language modeling.
% , and then the representation from the pretrained model is ``fine-tuned'' to solve a supervised downstream task.
Overparameterized transformers are known to be universal approximators \citep{Yun2020Are},
suggesting their generalization performance ought to rely on useful biases or constraints imposed by the learning algorithm. Despite various attempts to study these biases in transformers \citep{rogers-etal-2020-primer, lovering2021predicting}, it remains an interesting open question what they are, or even how to characterize them in a way relevant to the domain of language.
% Because these models are so large, their success is tied to the question of how their capacity is restricted by training. Standard generalization bounds are vacuous for large neural networks\roy{this seems like a strong claim, can you cite anything here, or alternatively tone down a bit?}---

% In NLP, there has been interest in understanding the mechanisms that emerge in transformers from pretraining \citep{rogers-etal-2020-primer}.
% Work in NLP has evaluated whether models show biases toward universal properties of language thought to be necessary for language understanding. In particular, this has focused on hierarchical syntax\roy{this sentence was hard for me to parse, can you simplify?} \citep{linzen-etal-2016-assessing, gulordava-etal-2018-colorless} and semantic compositionality \citep{kim-linzen-2020-cogs}. A related strand of work suggests that neural NLP models may often be ``right for the wrong reasons'' \citep{mccoy-etal-2019-right}: i.e, they are biased to learning simple heuristics rather than more generalizable hierarchical behaviors.\roy{you are citing 4 papers of your future advisor in this paragraph:) perhaps diversify a bit?}\roy{I think we might be better off without this paragraph, which feels like related work more than intro. Perhaps add ``despite various attempts to study these inductive biases \cite{...}'' above before ``But it is an interesting open question''}

In this work, we take the perspective that thoroughly understanding the dynamics of gradient descent (GD) might clarify the linguistic biases of transformers, and the types of representations they acquire.
We start by making a potentially surprising empirical observation (\autoref{sec:norm-growth}): the parameter $\ell_2$ norm grows proportional to $\sqrt{t}$ (where $t$ is the timestep) during the training of T5 \citep{raffel2019t5} and other transformers.
We refer to the phenomenon of growing parameter norm during training as \emph{norm growth}.
Previous work has analyzed norm growth in simplified classes of feedforward networks \citep{li2019exponential, telgarsky2020directional}, but, to our knowledge, it has not been thoroughly demonstrated or studied in the more complicated and practical setting of transformers.

Our main contribution is analyzing the \emph{effect} of norm growth on the representations within the transformer (\autoref{sec:effect}), which control the network's grammatical generalization.
With some light assumptions, we prove that any network where the parameter norm diverges during training approaches a \emph{saturated} network \citep{merrill2020aformal}: a restricted network variant whose discretized representations are understandable in terms of formal languages and automata.
% Theoretically, transformers increasingly approximate saturated networks as their norm grows.
Empirically, we find that internal representations of pretrained transformers approximate their saturated counterparts, but for randomly initialized transformers, they do not. This suggests that the norm growth implicit in training guides transformers to approximate saturated networks, justifying studying the latter \citep{merrill2019sequential} as a way to analyze the linguistic biases of NLP architectures and the structure of their representations. 

Past work \citep{merrill2019sequential, bhattamishra2020ability} reveals that
saturation permits two useful types of attention heads within a transformer: one that locally targets a small number of positions, and one that attends uniformly over the full sequence, enabling an ``average'' operation. Empirically, we find that both of these head types emerge in trained transformer language models. These capabilities reveal how the transformer can process various formal languages, and could also suggest how it might represent the structure of natural language.
Combined, our theoretical and empirical results shed light on the linguistic inductive biases imbued in the transformer architecture by GD, and could serve as a tool to analyze transformers, visualize them, and improve their performance.

Finally, we discuss potential causes of norm growth in \autoref{sec:norm-dynamics}. We prove transformers are approximately \textit{homogeneous} \cite{telgarsky2020directional}, a property that has been extensively studied in deep learning theory. With some simplifying assumptions, we then show how homogeneity might explain the $\sqrt{t}$ growth observed for T5.\footnote{Code available at \url{https://github.com/viking-sudo-rm/norm-growth}.}

\section{Background and Related Work} \label{sec:related-work}

\subsection{GD and Deep Learning Theory}
% \roy{the subsection title should include GD somewhere, shouldn't it? deep learning theory feels too broad} \nascomment{maybe:  Theory of Deep Learning and GD, which is elegantly semi-parallel to the next subsection's title}
A simple case where deep learning theory has studied the generalization properties of GD is matrix factorization \citep{gunasekar2017implicit, arora2019implicit, razin2020implicit}. It has been observed that deep matrix factorization leads to low-rank matrix solutions. \citet{razin2020implicit} argued theoretically that this bias of GD cannot be explained as an implicit regularizer minimizing some norm. Rather, they construct cases where all parameter norms diverge during GD.

% In order to explain the inductive bias of training feedforward networks, recent work has analyzed the dynamics of training from a new perspective.
Similar ideas have emerged in recent works studying feedforward networks.
Analyzing biasless ReLU networks with cross-entropy loss, \citet{poggio2019theoretical, poggio2020complexity} show that the \textit{magnitude} ($\ell_2$ norm) of the parameter vector continues to grow during GD, while its \textit{direction} converges. \citet{li2019exponential} present a similar argument for \emph{scale-invariant} networks, meaning that scaling the parameters by a constant does not change the output.
% This property leads the norm to monotonically increase during training.
Studying \emph{homogeneous} networks, \citet{telgarsky2020directional} show that the gradients become \emph{aligned} as $t \to \infty$, meaning that their direction converges to the parameter direction. This means the norm will grow monotonically with $t$.
The perspective developed by these works challenges the once conventional wisdom that the parameters converge to a finite local minimum during GD training. Rather, it suggests that GD follows a norm-increasing \textit{trajectory} along which network behavior stabilizes.
% However, such work makes simplifying assumption about feedforward networks--such as removing biases--and does not directly extend to more complicated networks like transformers.
These analyses motivate investigation of this trajectory-driven perspective of training.
% \roy{can we say that empirical results so far have been on small networks, and non of them for language tasks? (and we, for the first time, show this for real large scale NLP transformers?)0}

From a statistical perspective, work in this vein has considered the implications of these training dynamics for margin maximization \citep{poggio2019theoretical, nacson2019lexicographic, lyu2019gradient}. While these works vary in the networks they consider and their assumptions, they reach similar conclusions: GD follows trajectories diverging in the direction of a max-margin solution. As margin maximization produces a simple decision boundary, this property suggests better generalization than an arbitrary solution with low training loss. This point of view partially explains why growing norm is associated with better generalization performance.

% \will{Should we also cite the claim in the backpropaganda paper that low norm is actually not better, i.e., conventional understanding of weight decay is flawed?}

\subsection{NLP and Formal Language Theory}

Norm growth has another interpretation for NLP models. Past work characterizes the capacity of infinite-norm networks in terms of formal languages and automata theory.
\citet{merrill2019sequential} and \citet{merrill2020aformal} propose \emph{saturation}, a framework for theoretical analysis of the capacity of NLP architectures. A network is analyzed by assuming it \emph{saturates} its nonlinearities, which means replacing functions like $\sigma$ and $\tanh$ with step functions. This is equivalent to the following definition:
\begin{definition}[Saturation; \citealp{merrill2020aformal}] \label{def:saturation}
Let $f(x; \theta)$ be a neural network with inputs $x$ and weights $\theta$.
The \textit{saturated network} $\sat f(x; \theta)$ is\footnote{The limit over $f$ is taken pointwise. The range of $\sat f$ is $\mathbb{R}$.}
\begin{equation*}
    \sat f(x; \theta) = \lim_{c \rightarrow \infty} f(x; c\theta) ,
\end{equation*}
where the limit exists, and undefined elsewhere.
\end{definition}
Saturation reduces continuous neural networks to discrete computational models resembling automata or circuits, making some kinds of formal linguistic analysis easier. For many common architectures, the saturated capacity is known to be significantly weaker than the full capacity of the network with rational-valued weights \citep{merrill2019sequential}, which, classically, is Turing-complete for even simple RNNs \citep{siegelmann1992turing}.
% \footnote{Geometrically, the saturated capacity corresponds to a set of limit points in parameter space.}

For example, one can hand-construct an RNN or LSTM encoding a stack in its recurrent memory \citep{kirov2012fractals}.
Stacks are useful for processing compositional structure in linguistic data \citep{chomsky1956three}, e.g., for semantic parsing.
However, a saturated LSTM does not have enough memory to simulate a stack \citep{merrill2019sequential}. Rather, saturated LSTMs resemble classical counter machines \citep{merrill2019sequential}: automata limited in their ability to model hierarchical structure \citep{merrill2020linguistic}.
Experiments suggest that LSTMs trained on synthetic tasks learn to implement counter memory \citep{weiss2018, suzgun2019lstm}, and that they fail on tasks requiring stacks and other deeper models of structure \citep{suzgun2019memory}.
Similarly, \citet{shibata2020lstm} found that LSTM language models trained on natural language data acquire saturated representations approximating counters.

Recent work extends saturation analysis to transformers \citep{merrill2019sequential, merrill2020aformal}. Saturated attention heads reduce to generalized hard attention, where the attention scores can tie. In the case of ties, the head output averages the positions with maximal scores.\footnote{\citet{Hahn_2020} identified weaknesses of \emph{strictly} hard attention, which is weaker than saturated attention.} While their power is not fully understood, saturated transformers can implement a counting mechanism similarly to LSTMs \citep{merrill2020aformal}. In practice, \citet{bhattamishra2020ability} show transformers can learn tasks requiring counting, and that they struggle when more complicated structural representations are required.
\citet{ebrahimi-etal-2020-self} find that attention patterns of certain heads can emulate bounded stacks, but that this ability falls off sharply for longer sequences.
Thus, the abilities of trained LSTMs and transformers appear to be predicted by the classes of problems solvable by their saturated counterparts. \citet{merrill2020aformal} conjecture that the saturated capacity might represent a class of tasks implicitly learnable by GD, but it is unclear \emph{a priori} why this should be the case. This work aims to put this conjecture on more solid theoretical footing: we argue that approximate saturation arises in transformers as a result of norm growth during training.\footnote{This relates to \citet{correia-etal-2019-adaptively}, who modify the transformer to facilitate approximately sparse attention. In contrast, we will show that approximate sparsity (i.e., saturation) arises implicitly in \emph{standard} transformers.}
% \roy{can we say something even stronger somewhere (perhaps we are saying this somewhere:)? the conclusion I take from this paragraph is that in theory transformers can learn anything, but in practice their inductive bias leads them to a more limited capacity (of saturated transformers?). This feels like a strong claim that we don't really prove, but could hypothesize?}

%% file: sections/observation.tex
\section{Norm Growth in Transformers} \label{sec:norm-growth}

\begin{figure*}[ht]
    \centering
    \includegraphics[width=\columnwidth]{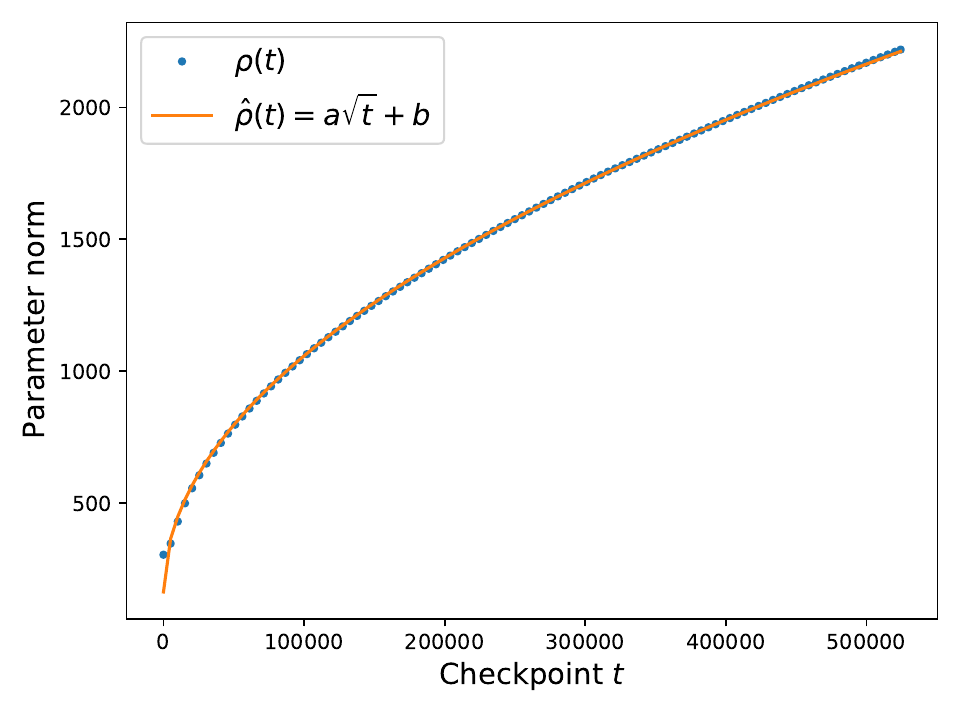}
    \includegraphics[width=\columnwidth]{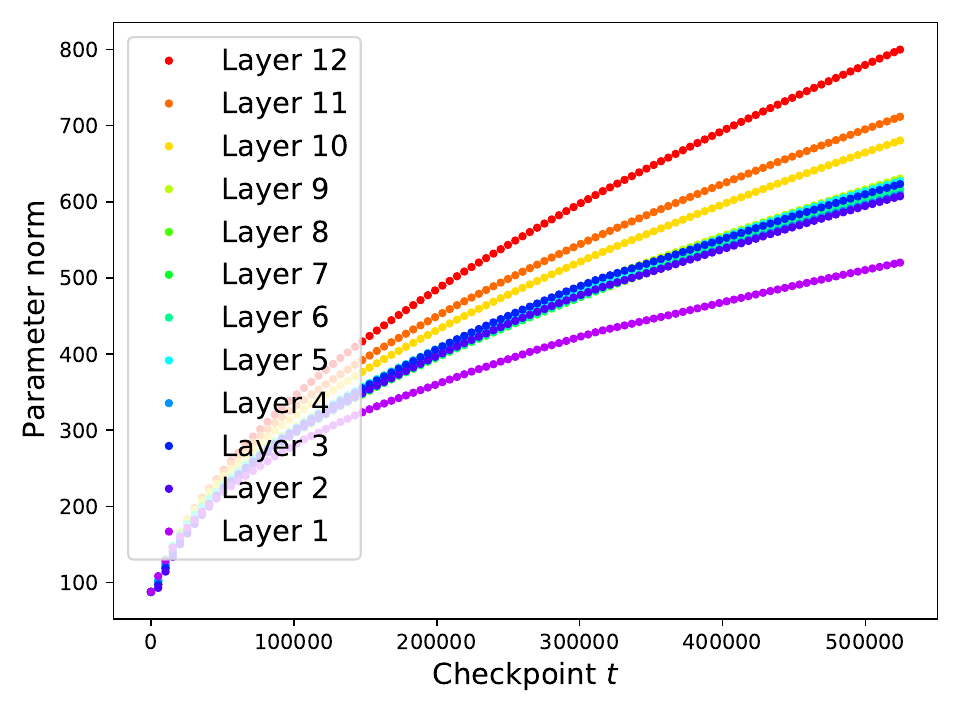}
    \includegraphics[width=\columnwidth]{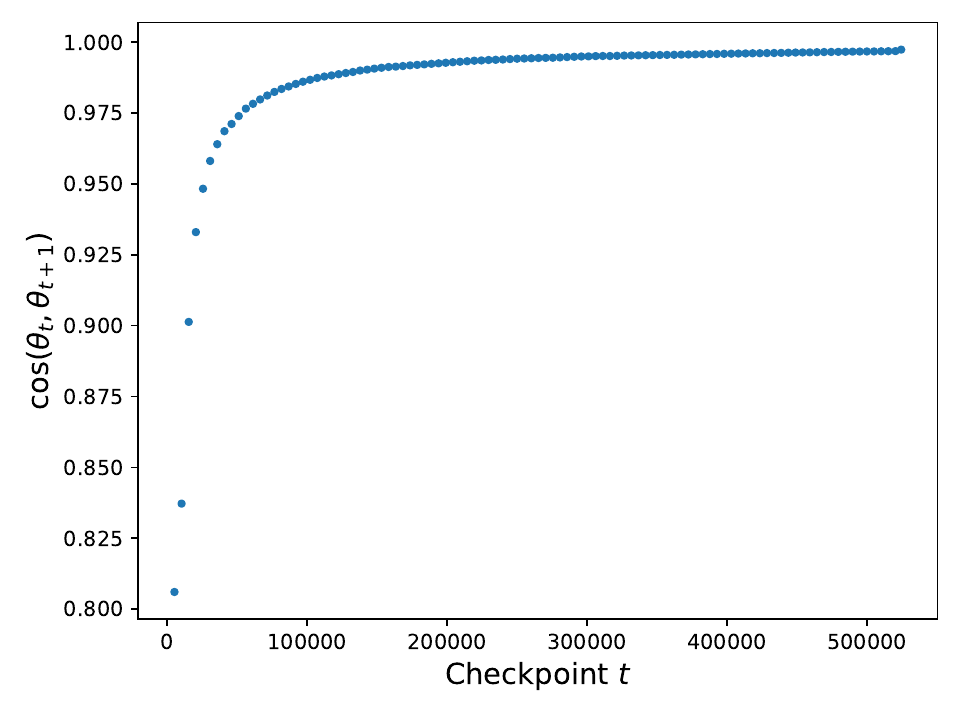}
    \includegraphics[width=\columnwidth]{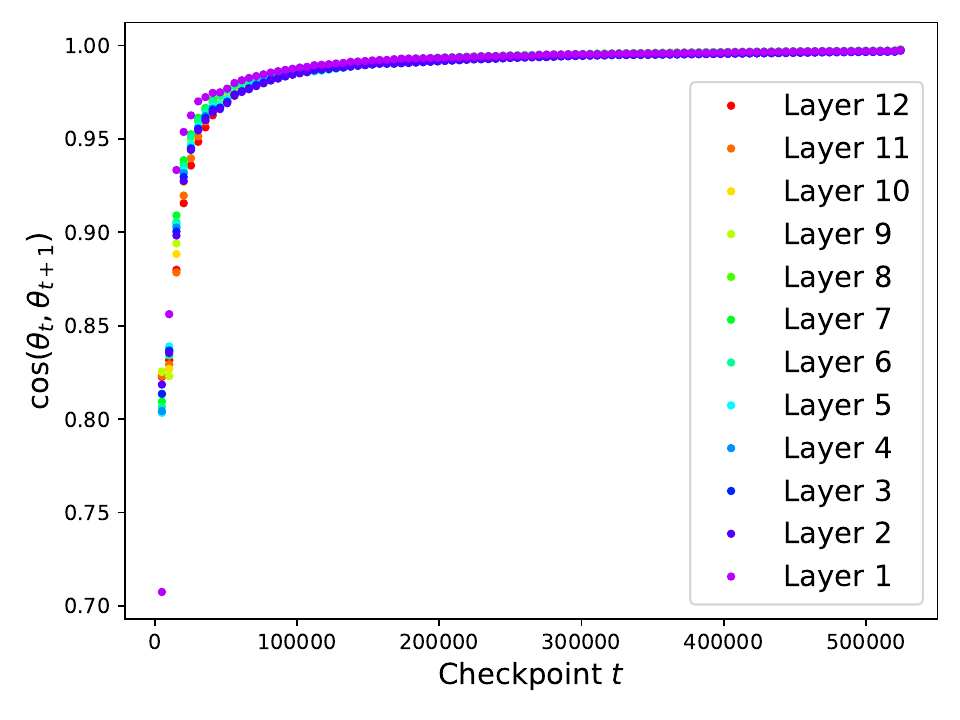}
    \caption{\textbf{Top}: Norm growth during T5 pretraining, with a coefficient $r^2 = 1.00$. The right is broken down by layer. \textbf{Bottom}: cosine similarity between subsequent parameter checkpoints.}
    \label{fig:t5-param-growth}
    % \vspace{-1em}
\end{figure*}

We start with the observation that the parameter $\ell_2$ norm grows during training for practical transformer language models. We first consider the parameter norm of $104$ historical checkpoints from T5-base \citep{raffel2019t5} pretraining, a 220M parameter model, which was trained using the AdaFactor optimizer \citep{adafactor}.
Further details are in \autoref{sec:details}.

%\paragraph{Norm growth}
\autoref{fig:t5-param-growth} shows that the T5 norm follows a $\sqrt{t}$ trend, where $t$ is time in training steps.
% The $\ell_2$ parameter norm over time is plotted in the top left \autoref{fig:t5-param-growth}, following a growth trend $\propto \sqrt{t}$.
The top right of \autoref{fig:t5-param-growth} breaks down the growth trend by layer. Generally, the norm grows more quickly in later layers than in earlier ones, although always at a rate proportional to $\sqrt{t}$.\footnote{We encourage future works that pretrain new transformer language models to track metrics around norm growth.}
Next, in the bottom row of \autoref{fig:t5-param-growth}, we plot the cosine similarity between each parameter checkpoint $\theta_{t+1}$ and its predecessor $\theta_t$. This rapidly approaches $1$, suggesting the ``direction'' of the parameters ($\theta_t / \norm{\theta_t}$) converges. The trend in directional convergence looks similar across layers.

We also train smaller transformer language models with 38M parameters on Wikitext-2 \citep{merity2016pointer} and the Penn Treebank (PTB;  \citealp{marcus-etal-1993-building}). %Here also we find that the norm grows consistently over training.
We consider two variants of the transformer: pre-norm and post-norm, which vary in the relative order of layer normalization and residual connections \citep[cf.][]{xiong2020layer}. Every model exhibits norm growth over training.\footnote{\textbf{Erratum:} In a previous paper version, this footnote reported perplexity numbers that we found to be irreproducible, as they were likely obtained with a non-standard truncated version of the PTB dataset. We have thus removed them.}
% \footnote{The post-norm transformer achieves $115.79$ perplexity on Wikitext-2 and $96.24$ on PTB. On the other hand, the pre-norm transformer reaches $66.35$ on Wikitext-2 and $26.16$ on PTB, slightly outperforming \citet{lm-with-trans}. This is consistent with previous findings \citep{xiong2020layer} showing advantages of pre-norm over post-norm.}

Combined, these results provide evidence that the parameter norm of transformers tends to grow over the course of training.
In the remainder of this paper, we will discuss the implications of this phenomenon for the linguistic biases of transformers, and then discuss potential causes of the trend rooted in the optimization dynamics.

%% file: sections/effects.tex
\section{Effect of Norm Growth} \label{sec:effect}

\autoref{sec:norm-growth} empirically documented that the parameter norm grows proportional to $\sqrt{t}$ during T5 pretraining.
% In \autoref{sec:norm-dynamics}, we explored the dynamics underlying norm growth, proposing a high-level explanation using results from deep learning theory.
Now, we move to the main contribution of our paper: the implications of norm growth for understanding transformers' linguistic inductive biases.
In particular, \autoref{thm:saturation-main} says uniform norm growth across the network guides GD towards saturated networks. Thus, saturation is not just a useful approximation for analyzing networks, but a state induced by training with enough time.
% \begin{theorem} \label{thm:saturation-main}
% Let $\theta_t \in \mathbb{R}^n$ be the parameter vector at train step $t$ for a network $f(x; \theta_t)$. Let $\theta_t^i$ denote a specific scalar parameter.
% For each $i, j$, as $t \rightarrow \infty$, assume that
% $\abs{\theta_t^i} \rightarrow \infty$ and $\theta_t^i / \theta_t^j$ converges.
% Then $f$ converges to a saturated network as $t \rightarrow \infty$.
% \end{theorem}
% \begin{corollary}
% If $\delta_t$ is aligned to $\theta_t$ as $t \to \infty$,
% % and $f(x; \theta_t)$ is approximately $k$-homogeneous (with $k\geq1$),
% then $f$ converges to a saturated network.
% \end{corollary}
\begin{theorem}[Informal] \label{thm:saturation-main}
Let $\theta_t \in \mathbb{R}^n$ be parameters at step $t$ for $f(x; \theta_t)$.
If every scalar parameter $\theta_t^i$ diverges at the same rate up to a constant, then $f$ converges pointwise to a saturated network.
% in function space.
\end{theorem}
The proof is in \autoref{sec:uniform}.
\autoref{thm:saturation-main} assumes not just norm growth, but uniform norm growth, meaning no parameter can asymptotically dominate any other. Notably, uniform growth implies directional convergence.
Accepting uniform growth for a given training regimen, we expect transformers to converge to saturated networks with infinite training.
Based on \autoref{sec:norm-growth}, the T5 norm appears to grow $\propto \sqrt{t}$ uniformly across the network, suggesting the uniform growth condition is reasonable.
As we will discuss later in \autoref{sec:norm-dynamics}, we expect the growth trend to depend heavily on the learning rate schedule.
% \roy{I got a little lost in the last paragraph. Here is what I think you are trying to say: a. We need two conditions: 1. norm growth. 2. Uniform growth. b. Fig1 shows empirically both of them hold. c.
% as a result we should expect transformers to saturate. Is this an accurate description?}

% \begin{figure}
%     \centering
%     \includegraphics[width=.7\columnwidth]{images/main_fig.pdf}
%     \label{fig:main_fig}
%     \caption{Informal 2D visualization of \autoref{thm:saturation-main}. Growing parameter magnitudes leads to saturated representations.}
% \end{figure}

\begin{figure*}[ht]
    \centering
    \includegraphics[width=\columnwidth]{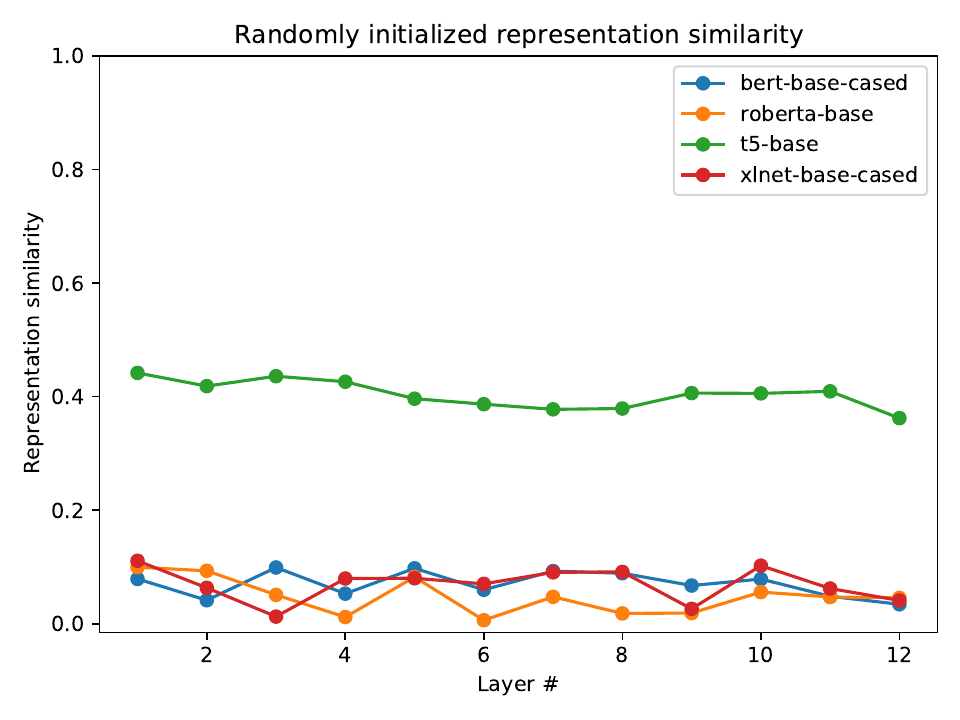}
    \includegraphics[width=\columnwidth]{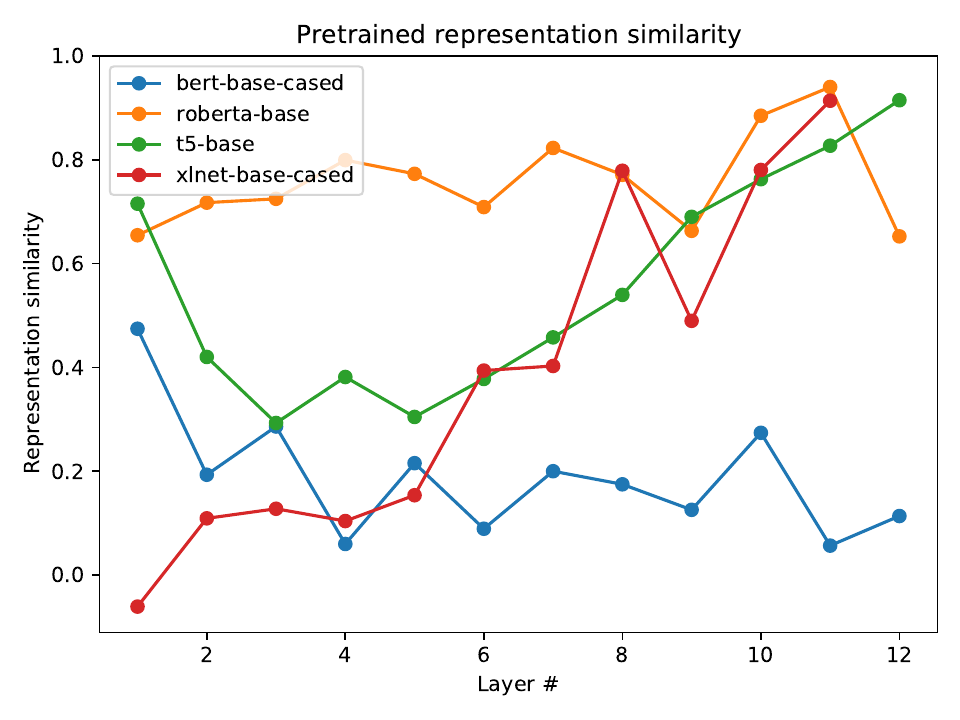}
    \caption{Cosine similarities of the unsaturated and saturated ($c=1{,}000$) transformer representations, by layer. We compare randomly initialized transformers (\textbf{left}) to pretrained ones (\textbf{right}).}
    \label{fig:rep-similarity}
\end{figure*}

\subsection{Saturated Transformers}\label{sec:sat_transformers_empirical}

Having established that norm growth should lead to saturation, we now empirically measure the saturation levels in T5 and other transformer models.

\paragraph{Large transformers are highly saturated.} \label{sec:rep-sim}
Since $\norm{\theta_t}$ empirically grows during training, we expect high cosine similarity between the representations in trained networks and saturated representations. We estimate this as the cosine similarity between $f(x; \theta)$ and $f(x; c\theta)$ for some large $c$ (in practice, $1{,}000$). 
We consider the ``base'' versions of pretrained BERT, RoBERTa, T5, and XLNet
(pretrained on masked language modeling),
and compute the mean saturation over $100$ input sentences from the Brown corpus \citep{francis1989manual}. To match standard practice, each sentence is truncated at $512$ word pieces. \autoref{fig:rep-similarity} plots the similarity for each layer of each model. We compare the pretrained transformers against a randomly initialized baseline.
For every model type, the similarity is higher for the pretrained network than the randomly initialized network, which, except for T5, is ${\sim}0$. For T5 and XLNet, the similarity in the final layer is ${\geq}0.9$, whereas, for RoBERTa, the final similarity is $0.65$ (although $0.94$ in the penultimate layer).
For T5 and XLNet, similarity is higher in later layers, which is potentially surprising, as one might expect error to compound with more layers. This may relate to the fact that the norm grows faster for later layers in T5. One question is why the similarity for BERT is lower than these models. As RoBERTa is architecturally similar to BERT besides longer training, we hypothesize that RoBERTa's higher similarity is due to longer pretraining.

% \begin{table}[t]
%     \centering
%     \begin{tabular}{|l|cc|cc|}
%         % DEV DEV DEV
%         % \hline
%         % & \multicolumn{2}{c|}{\textit{Wikitext-2}} & \multicolumn{2}{c|}{\textit{PTB}} \\
%         % Model & ppl & sat & ppl & sat \\
%         % \hline
%         % Pre-norm & 75.02 & 1.00 & 30.56 & 1.00 \\
%         % Post-norm & 127.61 & 1.00 & 101.14 & 1.00 \\
%         % \hline
        
%         \hline
%         & \multicolumn{2}{c|}{\textit{Wikitext-2}} & \multicolumn{2}{c|}{\textit{PTB}} \\
%         Model & ppl & sat & ppl & sat \\
%         \hline
%         Pre-norm & 66.35 & 1.00 & 26.16 & 1.00 \\
%         Post-norm & 115.79 & 1.00 & 96.24 & 1.00 \\
%         \hline
%     \end{tabular}
%     \caption{Test perplexity (ppl) and saturation (sat) for transformer language models trained on Wikitext-2 and the PTB. We consider two transformer variants: pre-norm and post-norm \citep{xiong2020layer}.
%     Further experimental details are in Appendix F.
%     }
%     \label{tab:lm-results}
% \end{table}

\paragraph{Small transformers reach full saturation.} Each of the transformers trained on Wikitext-2 and PTB reached a saturation level of $1.00$.
It is unclear why these models saturate more fully than the pretrained ones, although it might be because they are smaller.\footnote{Qualitatively, we observed that $^*$-small transformers tended to be more saturated than the $^*$-base models.}
For our LMs, the feedforward width ($512$) is less than for T5-base, while the encoder depth and width are the same.
Other possible explanations include differences in the initialization scheme, optimizer, and training objective (masked vs. next-word modeling).
See \autoref{sec:details} for full hyperparameters.
% \footnote{Although, since XLNet is autoregressive, we believe the training objective alone cannot explain the difference in saturation.}
% \will{This is not exactly true. XLNet is some weird variant of autoregressive training.}

%% file: sections/sat_transformers.tex
\begin{figure*}
    \centering
    \includegraphics[width=\columnwidth]{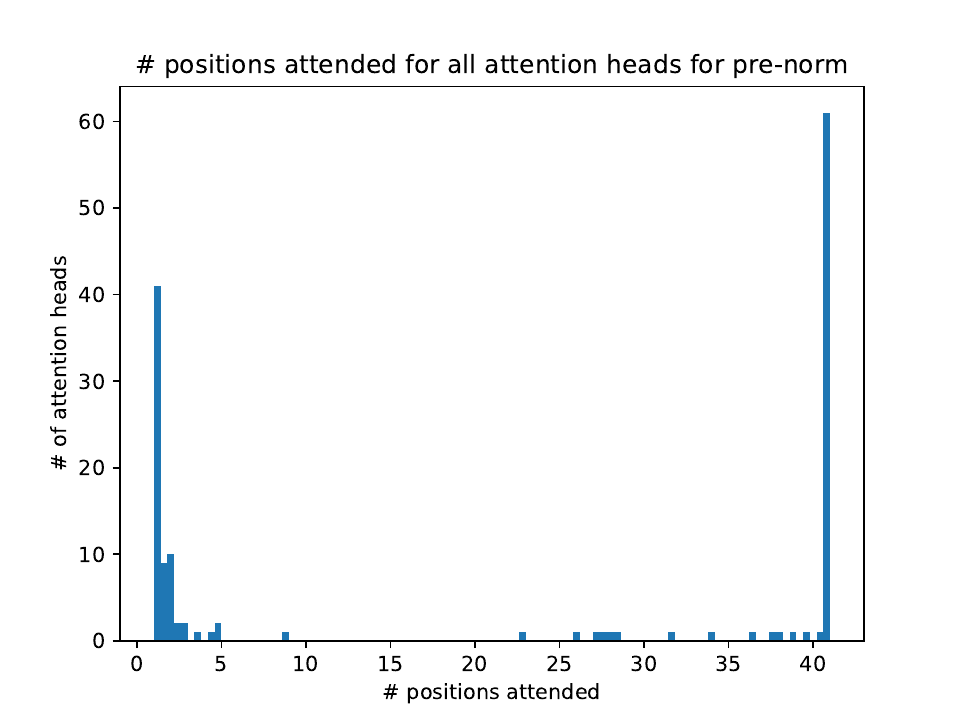}
    \includegraphics[width=\columnwidth]{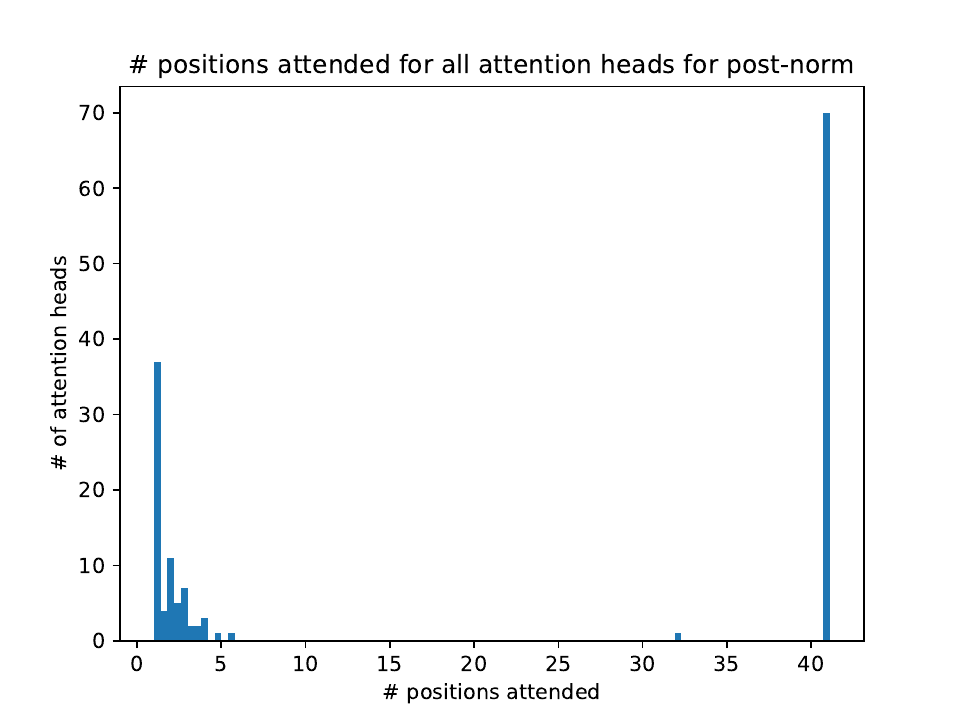}
    \caption{Distribution of the number of positions attended to for all heads in the PTB language models. The \textbf{left} plot is pre-norm, and the \textbf{right} is post-norm. Values are averaged over 200 sentences from the development set.}
    \label{fig:attn-dists}
    % \vspace{-1em}
\end{figure*}

\subsection{Power of Saturated Attention} \label{sec:sat_transformers}
% \subsection{Computational Capabilities of Saturated Attention} \label{sec:sat_transformers}

We have shown that transformer training increases the parameter norm (\autoref{sec:norm-growth}), creating a bias towards saturation (\autoref{sec:sat_transformers_empirical}). Now, we discuss the computational capabilities of saturated transformers, and empirically investigate how they manifest in pretrained transformers.
What computation can saturated transformers perform? We review theoretical background about saturated attention, largely developed by \citet{merrill2019sequential}.
Let $H$ (sequence length $n$ by model dimension $d$) be the input representation to a self attention layer. We assume a standard self attention mechanism with key, query, and value matrices $K, Q, V.$\footnote{To simplify presentation, we omit bias terms.}
Saturated attention resembles standard attention where softmax is constrained to a generalization of ``argmax'' \citep{merrill2019sequential}:
% \citet{merrill2019sequential} shows that a saturated self attention head ${\sat}\selfattn$ reduces to\roy{it is easy to miss the difference between the vanilla attention and the saturated one (I think it is only replacing softmax with argmax? better say it explicitly}
\begin{equation*}
    \sat \selfattn(H; Q, K, V) = \arg \max ( H Q K^\top H^\top ) HV .
\end{equation*}
We define this vectorized $\arg \max(A)$ as
\begin{align*}
    \mathcal M(A_i) &= \{ j \mid a_{ij} = \max_k a_{ik} \} \\
    \arg \max(A_i)_{j} &=
    \begin{cases}
        1 / \abs{\mathcal M(A_i)} & \textrm{if} \; j \in \mathcal M(A_i) \\
        0 & \textrm{otherwise.}
    \end{cases}
\end{align*}
Crucially, in the case of ties, $\arg \max(A)$ returns a uniform distribution over all tied positions.
Saturated attention can retrieve the ``maximum'' value in a sequence according to some similarity matrix.
It is also capable of restricted counting \citep{merrill2020aformal}.
Formalizing these observations, we identify two useful computational operations that are reducible to saturated self attention: \emph{argmax} and \emph{mean}. Let $h_i$ represent the input representation at each time step $1 \leq i \leq n$.
% A saturated self attention head can implement the following:
\begin{compactenum}
    \item \textbf{Argmax:} Set $V = \mathrm{Id}$. Then the self attention mechanism computes a function recovering the element of $H$ that maximally resembles $h_i$ according to a quadratic form $M = KQ^\top$. If there is a tie for similarity, a uniform average of the maximal entries in $H$ is returned.
    % Formally the $\max$ operation is:
    \begin{equation*}
        \mathrm{argmax}(H; M) = \arg \max_j h_i M h^\top_j .
    \end{equation*}
    \item \textbf{Mean:} Parameterize the head to attend uniformly everywhere. Then the head computes a function taking a uniform average of values:
    \begin{equation}
        \mathrm{mean}(H; V) = \frac{1}{n} \sum_{j=1}^n Vh_j. \label{eq:sum}
    \end{equation}
\end{compactenum}
These constructions demonstrate some useful computational abilities of saturated transformers. Due to the summation in \eqref{eq:sum}, the mean operation (or near variants of it) can be used to implement counting, which allows recognizing languages like $a^nb^nc^n$ \citep{merrill2020aformal}. Empirically, \citet{bhattamishra2020ability} find trained networks can learn to recognize counter languages that rely on computing means, failing on more complicated languages like Dyck-2. Our findings partially justify why transformers can learn these languages: they lie within the capacity of \textit{saturated} transformers.

\subsection{Learned Attention Patterns}

Recall that the small language models trained in \autoref{sec:sat_transformers_empirical} reach $1.00$ saturation.
It follows that we can convert them to saturated transformers (by multiplying $\theta$ by a large constant $c$)
without significantly shifting the representations in cosine space.
% Full saturation lets us discretize the attention patterns in these networks, without substantially changing the performance of the network.
We will evaluate if the saturated attention heads manifest the argmax and mean constructions from \autoref{sec:sat_transformers}.

% Because these models reach a saturation of $1.00$,
% we can multiply their parameters by a large constant without significantly shifting their representations in cosine space.
% This transformation converts every attention head to saturated attention.
As discussed in \autoref{sec:sat_transformers}, saturated attention can parameterize both argmax and mean heads. An argmax head
% , which at time $i$ attends to positions ``similar'' to token $i$
should attend to a small number of positions. A mean head, on the other hand, attends uniformly over the full sequence. Are both patterns acquired in practice by our models? We plot the distribution of the number of positions attended to by each head in the saturated PTB models in \autoref{fig:attn-dists}. The distribution is bimodal, with one mode at $1$, and the other around $41$, representing the mean sequence length of a $83$-length encoder with positional masking to prevent lookahead.
The empirical mode around $1$ corresponds to heads that are argmax-like. The mode around $41$, on the other hand, corresponds to mean-like heads, since it implies uniform attention over the masked sequence. Thus, our analysis suggests that analogs of both types of attention heads theorized in \autoref{sec:sat_transformers} are acquired in transformers in practice.
In the pre-norm transformer, which performs substantially better, there are also a small number of heads lying between the two modes. We defer the investigation of the function of these heads to future work.

%% file: sections/causes.tex
\section{Explanation for Norm Growth} \label{sec:norm-dynamics}

We have documented norm growth in T5 and other transformers (\autoref{sec:norm-growth}) and showed how it induces partial saturation in their representations (\autoref{sec:effect}).
This section points towards an understanding of \emph{why} the parameter norm grows over the course of training, grounded in results about norm growth from deep learning theory. We do not analyze specific optimizers directly; instead, we analyze norm growth within simplified models of training dynamics taken from the literature. We then evaluate how these candidate dynamics models fit T5's training.

\subsection{Setup}
Let $\delta_t \in \mathbb{R}^n$ denote the optimizer step at time $t$, i.e., $\delta_t = \theta_{t+1} - \theta_t$. We write $\eta_t$ for the learning rate at $t$.\footnote{Without loss of generality, the arguments presented here can be seen as applying to an individual parameter in the network, or the vector of all concatenated network parameters.} Let $\nabla_{\theta_t} L$ denote the gradient of the loss with respect to $\theta_t$.
By GD, we refer to the update
$\delta_t = - \eta_t \nabla_{\theta_t} L$.\footnote{Note that, in practice, T5 was trained with AdaFactor, whereas the setup in this section assumes simpler optimizers.}
% \begin{equation}
%     \delta_t = - \eta_t \nabla_{\theta_t} L .
% \end{equation}
In contrast, we will use the term \emph{gradient flow} to refer to its continuous relaxation, specified by an analogous differential equation:
\begin{equation*}
    \dv{\theta_t}{t} = -\eta_t \nabla_{\theta_t} L .
\end{equation*}

\subsection{Homogeneity}

We will rely on properties of homogeneous networks, a class of architectures well-studied in deep learning theory \citep{telgarsky2020directional}.
\begin{definition}[Homogeneity]
A function $f(x; \theta)$ is \textit{$k$-homogeneous} in $\theta$ iff, for all $c \geq 0$, $f(x; c\theta) = c^kf(x; \theta)$.
% Note that scale invariance is equivalent to $0$-homogeneity.
We further say that $f$ is homogeneous iff there exists some $k$ such that $f$ is $k$-homogeneous.
\end{definition}
Many common components of modern neural networks are homogeneous \citep{li2019exponential}. Furthermore, as various computations within a neural network preserve homogeneity (\autoref{sec:approx-homogeneity}), some full networks are also homogeneous. An example of a fully homogeneous neural network is a feedforward ReLU network without bias terms.

\begin{figure}
    \centering
    \includegraphics[width=.48\textwidth]{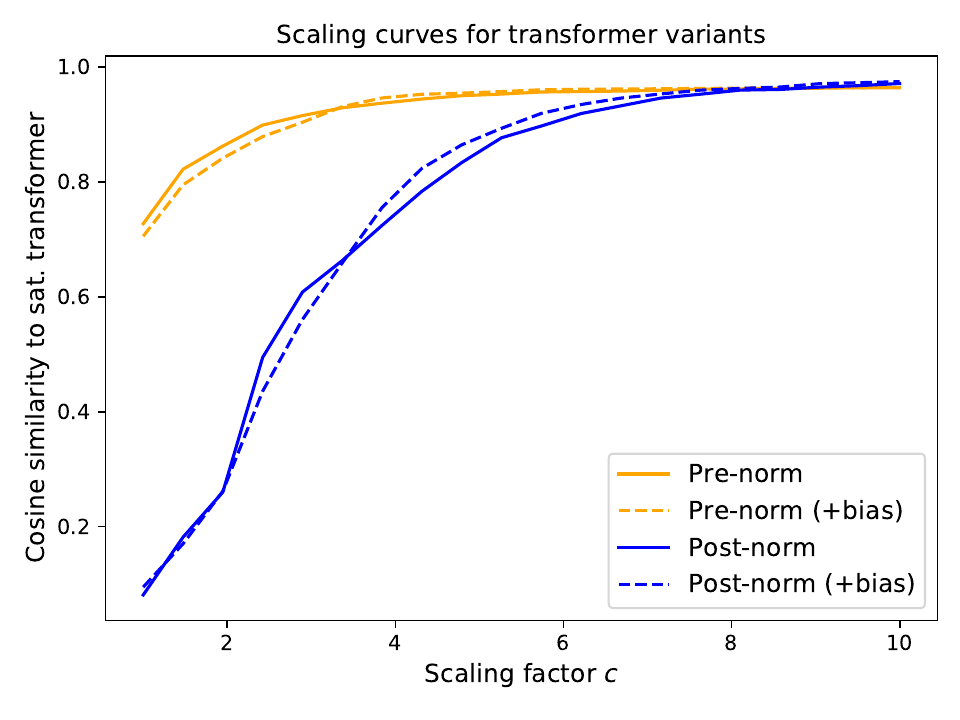}
    \caption{Approximate cosine similarity of $f(x; c \theta)$ to $\sat f(x; \theta)$ for randomly initialized transformers $f$.
    $\sat f(x; \theta)$ is approximated as in \autoref{fig:rep-similarity}.
    % A homogeneous function would have a value of $1$; thus, this can be taken as a measure of homogeneity.
    }
    \label{fig:trans-scaling-error}
\end{figure}

Why is homogeneity relevant for transformers?
Transformers are not homogeneous, but they are \emph{almost} homogeneous. We formalize this as:
\begin{definition}[Approx.~homogeneity] \label{def:approx-homo}
% A function
A scalar\footnote{A vector function is approximately $k$-homogeneous if this holds for all its elements.} function $f(x; \theta)$ is approximately $k$-homogeneous in $\theta$ iff
% its scaling error vanishes exponentially in $\norm{\theta}$, i.e.,
there exist $d, \rho$ s.t., for $c \geq 1$ and $\norm{\theta} \geq \rho$,
\begin{equation*}
    \abs{f(x; c\theta) - c^kf(x; \theta)} \leq
    \exp(-d \norm{\theta}) .
    % \exp(-\Omega(\norm{\theta})) .
\end{equation*}
\end{definition}
In other words, as $\norm{\theta}$ grows, $f$ approximates a homogeneous function with exponentially vanishing error.
In \autoref{sec:transformers}, we prove transformer encoders without biases are approximately $1$-homogeneous.
In \autoref{fig:trans-scaling-error}, we compare the cosine similarity of transformers with and without biases to their saturated variants, as a function of a constant $c$ scaling their weights. An approximately homogeneous function rapidly approach $1.0$ as $c$ increases. We find similar curves for transformers with and without biases, suggesting biasless transformers are similarly homogeneous to transformers with biases.\footnote{\citet{lyu2019gradient} find similar results for feedforward ReLU networks. It is an interesting puzzle why networks with biases appear similarly homogeneous to those without biases.}
% \roy{I think more explanation is needed as to why this is a measure of homogeneity. Currently you only briefly explain it in the image caption} we find empirically that transformers with biases are similarly homogeneous to biasless ones, by comparing their cosine similarities with respect to a saturated transformer.

Since multiplying two homogeneous functions adds their homogeneity, a transformer encoder followed by a linear classifier is approximately $2$-homogeneous.
A key property of homogeneous functions is Euler's Homogeneity Theorem: the derivative of a $k$-homogeneous function is $(k-1)$-homogeneous.
Thus, we will assume the gradients of the linear classifier output are roughly $1$-homogeneous, which under simple GD implies:
\begin{assumption} \label{ass:euler}
Let $\theta_t$ include all encoder and classifier parameters.
Let $\appropto$ mean ``approximately proportional to''.
For large enough $t$ during transformer training, $\norm{\delta_t} \appropto \eta_t \norm{\theta_t}$.
\end{assumption}
% In the left of \autoref{fig:t5-dir-align}, we find empirically that, after an initial chaotic phase, $\delta_t$ increases roughly monotonically. Visually the asymptotic rate ($\sqrt{t}$ vs $t$) is unclear, but we find a higher correlation ($r=0.68$) between $\norm{\delta_t}$ and $\norm{\theta_t}$ than between $\norm{\delta_t}$ and $\norm{\theta_t}^2$ ($r=0.66$), supporting \autoref{ass:euler}.

\begin{figure*}[ht]
    \centering
    \includegraphics[width=\columnwidth]{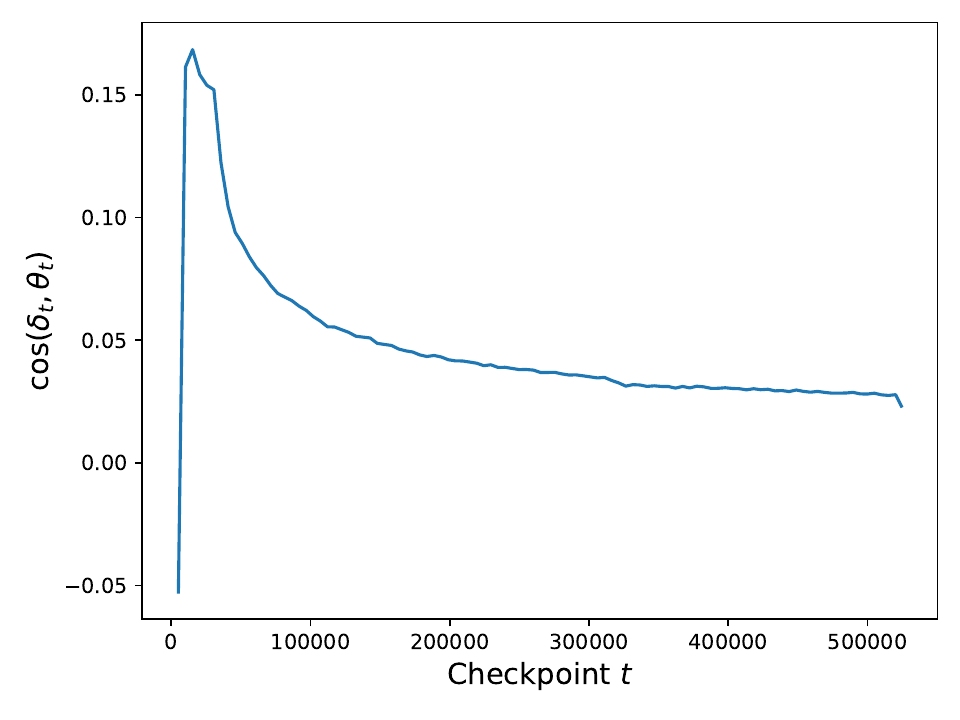}
    \includegraphics[width=\columnwidth]{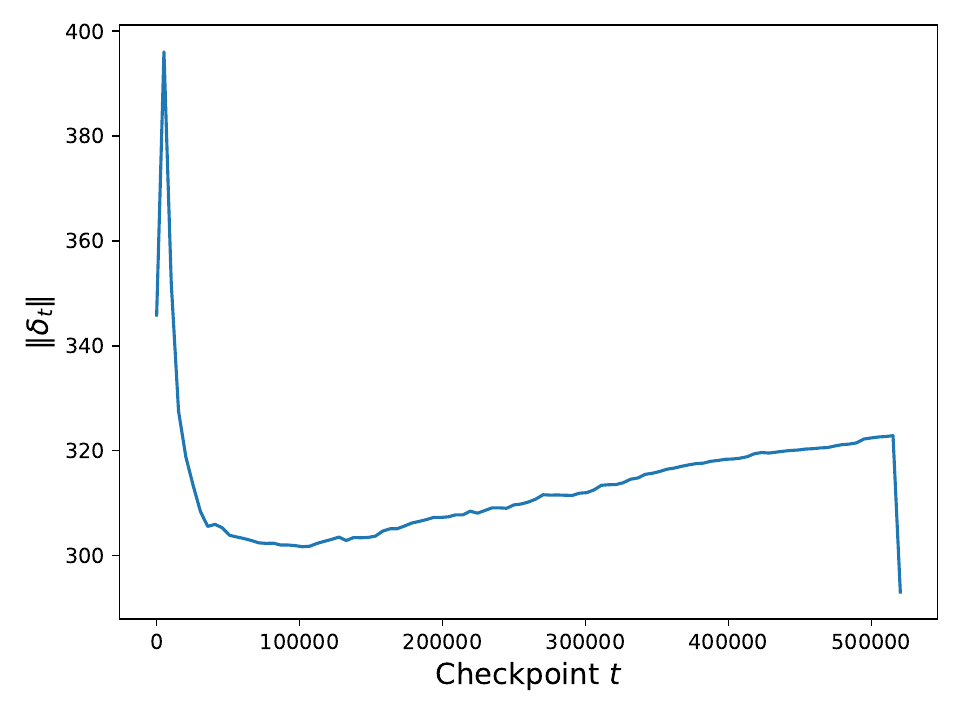}
    \caption{Alignment (cosine similarity of $\delta_t$ and $\theta_t$) and step size ($\norm{\delta_t}$) over training.
    % \roy{the drop towards the end of the right-hand figure might requite some explanation?}
    }
    \label{fig:t5-dir-align}
    % \vspace{-1em}
\end{figure*}

\subsection{Aligned Dynamics} \label{sec:aligned}
We now consider the first candidate dynamics model: \textbf{aligned dynamics} \citep{telgarsky2020directional}. Analyzing homogeneous networks with an exponential binary classification loss and gradient flow, \citet{telgarsky2020directional} show that the parameters converge in direction, and that the gradients become \emph{aligned}, meaning that $\theta_t^\top \cdot \delta_t \to \norm{\theta_t} \norm{\delta_t}$.
While it is unclear whether transformers will follow aligned dynamics, we entertain this as one hypothesis.
Under \autoref{ass:euler}, alignment implies
\begin{align*}
    \norm{\theta_t} &\approx \sum_{i=0}^t \norm{\delta_i} \appropto \int \eta_t \norm{\theta_t} \mathrm{d}t .
    % \therefore \norm{\theta_t} &\appropto \left( \int \eta_t \mathrm{d}t \right)^2 .
\end{align*}
With the $\eta_t = 1/\sqrt{t}$ schedule used by T5 \citep{raffel2019t5}, $\norm{\theta_t} \appropto \exp(\sqrt{t})$ (see \autoref{sec:int-aligned}). This is asymptotically faster than the observed $\sqrt{t}$ growth, suggesting an alternate dynamics might be at play.

\subsection{Misaligned Dynamics}
% We now formalize the dynamics of a transformer early in training. Because the network function is random at initialization, we model the gradient as a random vector at each time step, i.e., $\delta_t \sim \eta_t \cdot \mathcal N(0, \sigma^2)$. While this is not strictly true, we view it as an idealized model of the noisy dynamics early in training. Under these conditions, GD follows asimilar to scale-invariant networks \citep[cf.][]{li2019exponential}
Our second candidate model of training is \textbf{misaligned dynamics}, which follows largely from \citet{li2019exponential}.
This can be derived by assuming the gradients are misaligned (i.e., $\theta_t^\top \cdot \delta_t = 0$),
which hold for scale-invariant networks \citep{li2019exponential} and in expectation for random normal gradients.
Misalignment implies (derived in \autoref{sec:int-misaligned}):
% Here, we assume the gradients are overall well approximated by a random walk: $\delta_t \sim \eta_t \cdot \mathcal N(0, \sigma^2)$. This implies ``misalignment in expectation'', i.e., $\mathbb{E} \; \theta_t^\top \cdot \delta = 0$.\footnote{Misaligned dynamics also arise for $0$-homogeneous networks \citep{li2019exponential}, which also satisfy $\theta_t^\top \cdot \delta_t = 0$.} Assuming $\theta_0 \approx 0$,
% the expected norm is
\begin{equation} \label{eq:misaligned-int}
    \norm{\theta_t}^2 \appropto \sum_{i=0}^t \norm{\delta_i}^2 .
\end{equation}
We show in \autoref{sec:int-misaligned} that,
with the T5 learning rate ($\eta_t = 1 / \sqrt{t}$), \eqref{eq:misaligned-int} reduces to $\norm{\theta_t} \appropto \sqrt{t}$, as observed empirically for T5.
We now further test whether misaligned dynamics are a good fit for T5.

\subsection{Evaluation}

% We have identified two potential dynamics models for a transformer, resulting from different assumptions.
% While aligned growth 
% If the gradients are random and roughly independent, which we speculate resembles the beginning of training, then a network will follow \textbf{misaligned dynamics}, which lead to gradual norm growth. If the network is a homogeneous function, which is approximated after some amount of norm growth, then it will converge to \textbf{aligned dynamics}, which leads to a faster rate of growth.
% Aligned dynamics predicts growth $\appropto \sqrt{t}$ for T5, providing some weak evidence that the T5 dynamics may become approximately aligned for large $t$.

% We now measure $\norm{\delta_t}$ and alignment over the training of T5. For alignment, the metric we use is the cosine similarity of $\delta_t$ to $\theta_t$.
We measure the gradient alignment over the course of training T5. Our alignment metric is the cosine similarity of $\delta_t$ to $\theta_t$.
As shown on the left of \autoref{fig:t5-dir-align}, the alignment initially rapidly increases to ${\sim}0.15$, and then decays to near $0$. This supports the hypothesis that the T5 dynamics are misaligned, since the similarity is never high, and may be approaching $0$.

On the right of \autoref{fig:t5-dir-align}, we plot step size over training in order to evaluate the validity of \autoref{ass:euler}.
% The right side of \autoref{fig:t5-dir-align} plots the step size over the course of training.
At the beginning of training, a chaotic step size seems reasonable, as it is hard to predict the dynamics before approximate homogeneity takes hold. For large $t$,
\autoref{ass:euler} combined with the T5 learning rate schedule predicts step size should be roughly constant.\footnote{Since $\norm{\delta_t} \appropto \eta_t \norm{\theta_t} = \sqrt{t} / \sqrt{t} = 1$.}
This is not exactly what we find: for large $t$, $\norm{\delta_t}$ grows gradually with $t$. However, the absolute change in step size is small: $< 20$ across 220M parameters. Thus, we believe \autoref{ass:euler} is not unreasonable, though it would be interesting to understand what properties of the optimizer can explain the slight growth in step size.\footnote{
We believe the sharp drop in $\norm{\delta_t}$ at the final step is an artifact of the original recording of these checkpoints.}

\subsection{Weight Decay}
One feature of practical training schemes not considered in this section is weight decay.
% Neither candidate dynamics considers the effect of weight decay on norm growth.
% Weight decay penalizes large parameters by a coefficient $\lambda$. For example,
When applied to standard GD, weight decay can be written $\delta_t = -\eta_t \nabla_{\theta_t} L -  \lambda \theta_t$.
% Applied to standard GD:
% \begin{equation}
%     \delta_t = -\eta_t \nabla_{\theta_t} L -  \lambda \theta_t .
% \end{equation}
Intuitively, it might hinder norm growth if $\lambda$ is large.\footnote{Common wisdom says that weight decay improves generalization by keeping $\norm{\theta_t}$ small; however, recent work challenges the assumption that a bias towards small norm is beneficial \citep{Goldblum2020Truth}, suggesting the benefit of weight decay may arise from more subtle effects on the GD trajectory.} In \autoref{sec:weight-decay}, we report preliminary experiments testing the effect of weight decay on norm growth. Indeed, if $\lambda$ is set too large, weight decay can prevent norm growth, but within the standard range of values for $\lambda$, we find norm growth even in the face of weight decay. However, it is possible these results may change if the optimizer or other hyperparameters are varied.

% \roy{Looks much better. I wonder whether we need a discussion paragraph which summarizes the findings on the one hand, and the limitations on the other}

%% file: sections/conclusion.tex
\section{Conclusion} \label{sec:conclusion}
% We focus the rest of our analysis on presenting a hypothesis for why norm growth happens.

We empirically found that $\norm{\theta_t}$ grows $\propto \sqrt{t}$ during T5 pretraining---a fact that may be caused by the approximate homogeneity of the transformer architecture.
We proved that norm growth induces saturation,
and then showed empirically that T5 and other large transformers become approximately saturated through their pretraining.
Examining highly saturated transformer language models, we found the attention heads largely split between two distinct behaviors that can be roughly interpreted as argmax and mean operations.
While we lack a precise formal characterization of ``semi-saturated'' transformers, we conjecture their capacity resembles that of the saturated models.
Thus, we believe further analyzing the capabilities of saturated attention may clarify the linguistic biases that emerge in transformers through training, and the mechanisms they use to represent linguistic structure.

% Finally, we evaluate to what degree different theories of norm growth in idealized neural networks explain the norm growth observed in T5.
% In \autoref{sec:norm-dynamics} we develop a basic theoretical perspective for why the parameter norm should grow during training based around the approximate homogeneity of transformers.
% . Accuracy, learning rate, and weight decay are identified as important factors modulating norm growth. Developing a more complete theoretical story is left to future work.
% We leverage the fact that transformers are approximately homogeneous to show that, with binary cross-entropy loss, the potential for norm growth increases with accuracy on the training set, both in a simplified and formal model. We present the intuition that there exists a ball in parameter space of \textit{gradient projection equilibrium}, which expands as accuracy increases. For multiclass loss, however, this theoretical story does not hold. Instead, the parameter magnitudes should continue to grow as long as the loss does not go to $0$, as observed for T5's pretraining.
% It is unclear whether the T5 norm would eventually be forced to converge by this ball during pretraining, or whether the theory is fundamentally different for a multiclass loss.

%% file: sections/exp_details.tex
\section{Experimental Details} \label{sec:details}

We provide experimental details for the small language models that we trained. The models were trained for 5 epochs, and the best performing model was selected based on development loss. Reported metrics were then measured on the held-out test set. We used our own implementation of the standard pre- and post-norm transformer architectures. We did not do any hyperparameter search, instead choosing the following hyperparameters:
\begin{compactitem}
    \item Batch size of $16$
    \item Model dimension of $768$
    \item Feedforward hidden dimension of $512$
    \item $12$ heads per layer
    \item $12$ layers
    \item AdamW optimizer with default PyTorch hyperparameters
    \item 0 probability of dropout
    \item Default PyTorch initialization
\end{compactitem}

\paragraph{Tokenization} For Wikitext-2, $3$ tokens in the whole test dataset were unattested in the training set (due to capitalization). To make our model compatible with unseen tokens, we replaced these tokens with \texttt{<unk>}, the same class that appeared for low frequency words at training time, when evaluating the final text perplexity. Due to the small number of tokens that were affected, the impact of this change should be negligible.

\paragraph{Compute} We estimate the experiments in this paper took several hundred GPU hours on NVIDIA A100 GPUs over the course of almost two years of on-and-off research time.

\paragraph{T5} We used the historical checkpoints of \texttt{bsl-0}, one of five T5-base models that was trained for the original paper \citep{raffel2019t5}.

\paragraph{Measuring Norms} As a systematic choice, all measurements of parameter norm include only \emph{encoder} parameters that are not scalars. We advise other researchers to follow the practice of excluding embedding parameters, as embedding parameters that are infrequently updated may obscure general trends in the network parameters.

%% file: sections/uniform.tex
\section{Norm Growth and Saturation} \label{sec:uniform}

\begin{theorem}[Formal version of \autoref{thm:saturation-main}]
Let $\theta_t \in \mathbb{R}^n$ be the parameter vector at train step $t$ for a network $f(x; \theta_t)$.
Assume that, as $t \to \infty$, there exists a scalar sequence $c(t) \to \infty$ and fixed vector $\theta' \in (\mathbb{R} \setminus \{0\})^n$ such that, for all $t$, $\theta_t \to \theta' \cdot c(t)$.
Then $f$ converges pointwise to a saturated network in function space.
\end{theorem}

\begin{proof}
\begin{equation*}
    \lim_{t \to \infty} f(x; \theta_t) = \lim_{t \to \infty} f(x; \theta' \cdot c(t)) .
\end{equation*}
Now, since $c(t) \to \infty$ and $\theta' \cdot c(t)$ contains no indeterminate elements, we can simplify this to
\begin{equation*}
    \lim_{c \to \infty} f(x; c\theta') = \sat f(x; \theta') .
\end{equation*}
\end{proof}

%% file: sections/approx_homogeneity.tex
\section{Approximate Homogeneity} \label{sec:approx-homogeneity}

In this section, we will further develop the notion of approximate homogeneity. We will prove that is consistent. In other words, every function can have at most one degree $k$ of approximate homogeneity. Next, we will show that the useful closure properties applying to full homogeneity also apply to partial homogeneity.

\begin{table}
    \centering
    \begin{tabular}{ |c|c|c| } 
    \hline
    Component & $k$ Input & $k$ Output \\
    \hline
    Linear & $k$ & $k+1$ \\ 
    Bias & $1$ & $1$ \\ 
    Affine & $0$ & $1$ \\
    LayerNorm & $k$ & $0$ \\ 
    LayerNorm + Affine & $k$ & $1$ \\
    ReLU & $k$ & $k$ \\
    Sum & $(k, k)$ & $k$ \\
    Product & $(k_1, k_2)$ & $k_1 + k_2$ \\
    \hline
    \end{tabular}
    \caption{Effects of network components on homogeneity shown by \citet{li2019exponential}. We write the ``$k$ Output'' homogeneity as a function of the ``$k$ Input'' homogeneity. These facts can be applied recursively to compute the homogeneity of a network. We will show that the same facts hold for \textit{approximate} homogeneity.}
    \label{tbl:homogeneity}
\end{table}

If $f(\theta)$ is approximately $k$-homogeneous (cf.~\autoref{def:approx-homo}), then $f(c\theta) = c^kf(\theta) + \epsilon$ for some error vector $\epsilon$ where, for each $i$, $\abs{\epsilon_i} \leq \exp(-d \norm{\theta}))$, for all $c$ and large enough $\norm{\theta}$. We use this $\epsilon$ notation throughout this section.

\subsection{Consistency}

We first prove that approximate homogeneity is consistent: in other words, if a function is both approximately $k_1$ and $k_2$-homogeneous, then $k_1 = k_2$. This is an important property for establishing approximate homogeneity as a meaningful notion.

\begin{lemma}
Let $k_1, k_2 \in \mathbb{N}$. Assume that $f$ is both approximately $k_1$ and $k_2$-homogeneous. Then $k_1 = k_2$.
\end{lemma}

\begin{proof}
If $f$ is both approximately $k_1$ and $k_2$-homogeneous, then we have vanishing terms $\epsilon_1$ and $\epsilon_2$ such that, for all $c$,
\begin{align*}
    f(c\theta) &= c^{k_1} f(\theta) + \epsilon_1 \\
    f(c\theta) &= c^{k_2} f(\theta) + \epsilon_2 .
\end{align*}
Subtracting both sides yields
\begin{align*}
    0 &= (c^{k_1} - c^{k_2}) f(\theta) + \epsilon_1 - \epsilon_2 \\
    \therefore \abs{c^{k_1} - c^{k_2}} &= \frac{\abs{\epsilon_1 - \epsilon_2}}{f(\theta)} .
\end{align*}
The right-hand side vanishes exponentially in $\norm{\theta}$ for all $c$, whereas the left-hand side grows with $c$ unless $k_1 = k_2$. Thus, to satisfy this equation for all $c$, it must be the case that $k_1 = k_2$. 
\end{proof}

\subsection{Closure Properties}

We now prove that effects of various functions on homogeneity explored by \citet{li2019exponential} also translate to approximate homogeneity. 

\begin{lemma}
$\relu$ preserves approximate $k$-homogeneity, i.e., let $f: \mathbb{R}^n \rightarrow \mathbb{R}$ be approximately $k$-homogeneous. Then $\relu \circ f$ is approximately $k$-homogeneous.
\end{lemma}
\begin{proof}
\begin{align*}
    \relu \big( f(c\theta) \big)
        &= \relu \big(c^k f(\theta) + \epsilon \big) \\
        &\leq \relu \big(c^k f(\theta) \big) + \abs{\epsilon} .
\end{align*}
Therefore,
\begin{equation*}
       \abs{\relu \big( f(c\theta) \big) - \relu \big(c^k f(\theta) \big)} \leq \abs{\epsilon} . 
\end{equation*}
Set $\epsilon' = \abs{\epsilon}$, showing $\relu \big( f(\theta) \big)$ is approximately $k$-homogeneous.
\end{proof}

\begin{lemma} \label{thm:homo-sum}
Let $f,g$ be vector-valued functions of $\theta$.
If $f$ and $g$ are approximately $k$-homogeneous, then $f + g$ is approximately $k$-homogeneous.
\end{lemma}
\begin{proof}
\begin{align*}
    f(c\theta) + g(c\theta)
        &= c^kf(\theta) + \epsilon_f + c^kg(\theta) + \epsilon_g \\
        &= c^kf(\theta) + c^kg(\theta) + \epsilon' ,
\end{align*}
\noindent where $\epsilon' = \epsilon_f + \epsilon_g$. Thus,
\begin{equation*}
    \abs{f(c\theta) + g(c\theta) - c^k \big(f(\theta) + g(\theta) \big)} \leq \epsilon' .
\end{equation*}
\end{proof}

\begin{lemma} \label{thm:homo-prod}
Let $f,g$ be vector-valued functions of $\theta$.
If $f$ and $g$ are approximately $k_f$ and $k_g$-homogeneous, then $f \cdot g$ is approximately  $(k_f + k_g)$-homogeneous.
\end{lemma}

\begin{proof}
\begin{align*}
    f(c\theta) \cdot g(c\theta)
        &= \big( c^k f(\theta) + \epsilon_f \big) \cdot \big( c^k g(\theta) + \epsilon_g \big) \\
        &= c^{k_f+k_g}f(\theta)g(\theta) + c^{k_f} f(\theta)\epsilon_g \\
        &\quad + c^{k_g} g(\theta)\epsilon_f + \epsilon_f\epsilon_g .
\end{align*}
We now rewrite the term $c^{k_f} f(\theta) \epsilon_g$ as
\begin{equation*}
    \frac{\theta g(x; \hat \theta)}{\exp ( -d \norm{\theta})} \leq \exp(- d' \norm{\theta}) .
\end{equation*}
Now, set $\epsilon' = \min(\exp ( -d \norm{\theta}), \epsilon_f \epsilon_g)$.
\begin{equation*}
    \abs{f(c\theta) g(c\theta) - c^{k_f+k_g}f(\theta)g(\theta)} \leq \epsilon' .
\end{equation*}
\end{proof}

The analogous results for linear transformation, bias, and affine transformation directly follow from the results for sum and product in \autoref{thm:homo-sum} and \autoref{thm:homo-prod}.

Finally, we show that layer norm converts a homogeneous function to approximately scale-invariant function. In order to be numerically stable, practical implementations of layer norm utilize a small tolerance term so that the denominator is never zero. We omit this practical detail from our analysis, instead defining the layer norm $\lnorm(x)$ for $x \in \mathbb{R}^n$ according to
\begin{align*}
    \mu(x) &= \frac{1}{n} \sum_{i=1}^n x_i \\
    \lnorm(x)_i &= \frac{x_i - \mu(x)}{\norm{x - \mu(x)}} .
\end{align*}

\begin{lemma}
Let $f$ be approximately $k$-homogeneous for some $k$. Then, $\lnorm(f)$ is approximately $0$-homogeneous.
\end{lemma}
\begin{proof}
Since addition preserves approximate $k$-homogeneity, mean (and difference to mean), preserve approximate $k$-homogeneity. Letting $C = c^k$, we can write
\begin{equation*}
    f(c\theta) - \mu(f(c\theta)) = C \big( f(\theta) - \mu(f(\theta)) \big) + \epsilon .
\end{equation*}
We now apply this to the definition of layer norm to get
\begin{align*}
    \lnorm(f(c\theta))_i
        &= \frac{f(c\theta)_i - \mu(f(c\theta))}{\norm{f(c\theta) - \mu(f(c\theta))}} \\
        &= \frac{C \big( f(\theta)_i - \mu(f(\theta)) \big) + \epsilon_i}{C \norm{ f(\theta) - \mu(f(\theta)) } + \epsilon} .
\end{align*}
We show that the difference between this and the unscaled layer norm goes to zero. To simplify notation, we now write $f = f(\theta)$, $\mu = \mu(f(\theta))$, and $\epsilon = \epsilon$ in the left-hand side below:
\begin{align*}
    &\abs{\lnorm(f(c\theta))_i - \lnorm(f(\theta)_i)} \\
        &= \abs{ \frac{C \big( f_i - \mu \big) + \epsilon_i}{C \norm{ f - \mu } + \epsilon} - \frac{ f_i - \mu}{\norm{ f - \mu }} } \\
        % &= \abs{ \frac{C (f_i - \mu)\norm{f - \mu} + \epsilon_i\norm{f - \mu} - C (f_i - \mu)\norm{f - \mu} - \epsilon(f_i - \mu)}{C \norm{f - \mu}^2 + \epsilon\norm{f - \mu}} } \\
        &= \abs{ \frac{ \epsilon_i\norm{f - \mu} - \epsilon(f_i - \mu)}{C \norm{f - \mu}^2 + \epsilon\norm{f - \mu}} } \\
        &= \abs{ \frac{ \epsilon_i - \epsilon v}{C \norm{f - \mu} + \epsilon} } \\
        &\leq \abs{ \frac{ \epsilon_i - \epsilon v}{\epsilon} }.
\end{align*}
\noindent for some $v \in \mathbb{R}^n$ which does not grow with $\norm{\theta}$. Thus, setting $\epsilon'$ to this final quantity satisfies the definition of approximate $0$-homogeneity, i.e. approximate scale invariance.
\end{proof}

\subsection{Saturating Activation Functions}

We show that the exponentially saturation activation functions $\sigma$, $\softmax$, and $\tanh$ are approximately scale-invariant in $x$, i.e. scaling $x$ has an exponentially diminishing effect on the output. 
We start by analyzing the simpler sigmoid, and then show that the same result holds for softmax. For completeness, we then present a proof for $\tanh$. We use $\Theta$ (not $\theta$) in the standard sense of asymptotic notation.

\begin{lemma} \label{thm:sigmoid-error}
The scaling error for $\sigma$ vanishes exponentially in the preactivation magnitude, i.e. for all $c \geq 1$,
\begin{equation*}
    \abs{\sigma(cx) - \sigma(x)} \leq \Theta(\exp(-\abs{x})) .
\end{equation*}
\end{lemma}
\begin{proof}
Assume without loss of generality that $x \neq 0$, as if this is the case, the error is $0$. When $x > 0$, we have
\begin{align*}
    \abs{\sigma(cx) - \sigma(x)}
        &= \sigma(cx) - \sigma(x) \\
        % &= 1 - \sigma(x) \\
        &\leq 1 - \sigma(\abs{x}) \\
        &= \frac{1}{\exp(\abs{x}) + 1} \\
        &= \Theta(\exp(-\abs{x})).
\end{align*}
When $x < 0$, we have
\begin{align*}
    \abs{\sigma(cx) - \sigma(x)}
        &= \sigma(x) - \sigma(cx) \\
        &\leq 1-\sigma(\abs{x}) + 0 \\
        &= \Theta(\exp(-\abs{x})).
\end{align*}
\end{proof}

\begin{lemma} \label{thm:softmax-error}
The elementwise scaling error for $\softmax$ vanishes exponentially in the preactivation norm, i.e. for all $c \geq 1$, $x \in \mathbb{R}^n$ s.t. $1 \leq i \leq n$,
\begin{equation*}
    \abs{\softmax(cx)_i - \softmax(x)_i} \leq \exp(-\Theta(\norm{x})) .
\end{equation*}
\end{lemma}
\begin{proof}
The proof closely follows that of \autoref{thm:sigmoid-error}, but is more involved.
We consider two cases: $x_i = \max(x)$, and $x_i \neq \max(x)$.
% We consider three cases: $x_i = \max(x)$, $x_i = \min(x)$, and $\min(x) < x_i < \max(x)$. We assume without loss of generality that $\min(x) \neq \max(x)$, since if this is true, softmax outputs the (scale-invariant) uniform distribution.

\paragraph{Case 1} $x_i = \max(x)$.

\begin{align*}
    &\abs{\softmax(cx)_i - \softmax(x)_i} \\
    &= \softmax(cx)_i - \softmax(x)_i \\
    &\leq 1 - \softmax(x)_i \\
    &= 1 - \frac{\exp(x_i)}{\sum_j \exp(x_j)} \\
    &\leq 1 - \frac{\exp(\max(x))}{\exp(\max(x)) + (n-1)\exp(\min(x))} \\
    &= 1 - \frac{1}{1 + (n-1)\exp(\min(x) - \max(x))}  \\
    &= 1 - \frac{1}{1 + \exp(\min(x) - \max(x) + d)} ,
\end{align*}
\noindent for some $d \in \mathbb{R}$. As this has the form of $\sigma$,
\begin{align*}
    \abs{\softmax(cx)_i - \softmax(x)_i} \\
        = 1 - \sigma(\Theta(\norm{x}))
        = \exp(-\Theta(\norm{x})).
\end{align*}

\paragraph{Case 2} $x_i \neq \max(x)$.
\begin{align*}
    &\abs{\softmax(cx)_i - \softmax(x)_i} \\
    &= \softmax(x)_i - \softmax(cx)_i \\
    &\leq 1 - \max(\softmax(x)) - 0 \\
    &= 1 - \softmax(\max(x)) ,
\end{align*}
\noindent which is identical to case 1.
\end{proof}

Finally, for completeness, we show that $\tanh$ exhibits the same property. The proof is very similar to sigmoid, following closely from the definition
\begin{equation*}
    \tanh(x) = \frac{\exp(2x) - 1}{\exp(2x) + 1} .
\end{equation*}

\begin{lemma} \label{thm:tanh-error}
The scaling error for $\tanh$ vanishes exponentially in the preactivation magnitude, i.e. for all $c \geq 1$,
\begin{equation*}
    \abs{\tanh(cx) - \tanh(x)} \leq \exp(-\Theta(\abs{x})) .
\end{equation*}
\end{lemma}

\begin{proof}
\begin{align*}
    \abs{\tanh(cx) - \tanh(x)}
        &\leq \abs{1 - \tanh(x)} \\
        &= 1 - \tanh(\abs{x}) \\
        &= 1 - \frac{\exp(2\abs{x}) - 1}{\exp(2\abs{x}) + 1} \\
        &= \frac{\exp(2\abs{x}) + 1 - \exp(2\abs{x}) + 1}{\exp(2\abs{x}) + 1} \\
        &= \frac{2}{\exp(2\abs{x}) + 1} \\
        &= \exp(-\Theta(\abs{x})) .
\end{align*}
\end{proof}

Thus, applying these functions to a homogeneous input produces an output that is approximately scale-invariant in the parameters $\theta$. Thus, these functions act similarly to layer norm, which maps homogeneous input to scale-invariant output. But what happens if the input is \textit{approximately} homogeneous, rather than strictly homogeneous? In this case, we show that the output is approximately scale-invariant assuming $\norm{\theta}$ is sufficiently large.

\begin{theorem} \label{thm:approx-softmax}
Let $f(x; \theta)$ be approximately $k$-homogeneous in $\theta$. Then
% there exists $\rho$ where, for all $\theta$ such that $\rho < \norm{\theta}$,
the following functions are approximately scale-invariant in $\theta$:
\begin{align*}
    g_{\sigma} &= \sigma \circ f \\% \sigma(f(x; \theta)) \\
    g_{\softmax} &= \softmax \circ f \\
    g_{\tanh} &= \tanh \circ f .
\end{align*}
\end{theorem}

\begin{proof}
If $f(x; \theta)$ is approximately $k$-homogeneous, then $f(x; c\theta) = c^k f(x; \theta) + \epsilon$ where $\norm{\epsilon} \leq \exp(-O(\norm{\theta}))$. Crucially, since $\epsilon$ vanishes for large norm, there is some $\rho$ where, for all $\theta$ such that $\rho < \norm{\theta}$:
\begin{align*}
    \sgn \big( c^k f(x; \theta) + \epsilon \big) &= \sgn \big( c^k f(x; \theta) \big) \\
    \arg \max \big( c^k f(x; \theta) + \epsilon \big) &= \arg \max \big( c^k f(x; \theta) \big) .
\end{align*}

Therefore, for $\theta$ such that $\norm{\theta} > \rho$, the bounds used in \autoref{thm:sigmoid-error}, \autoref{thm:softmax-error}, and \autoref{thm:tanh-error} hold for approximately homogeneous $f$. Thus, we can conclude that the output is approximately scale-invariant.
\end{proof}

%% file: sections/transformer.tex
\section{Transformers} \label{sec:transformers}
We introduce the notation ${\sim}k$-homogeneous to mean approximately $k$-homogeneous.
In this section, we show that the transformer encoder is ${\sim}1$-homogeneous. A transformer \citet{vaswani2017attention} is made up of three main components: an embedding layer, self attention sublayers, and feed-forward sublayers. Since the embedding layer is just a matrix multiplication, it is a $1$-homogeneous function of the input. Assuming the self attention and feed-forward sublayers have no bias terms, we show that they approximate functions preserving approximate $1$-homogeneity. As the full network is an initial embedding layer followed by these sublayers, the final output is ${\sim}1$-homogeneous. In the main paper, we discuss the connection between homogeneity and norm growth.

We base our analysis on the HuggingFace implementation\footnote{\url{https://huggingface.co/transformers/_modules/transformers/modeling_bert.html\#BertModel}} of BERT \citep{Wolf2019HuggingFacesTS}.
To aid analysis, we make some simplifying assumptions, which are discussed along with the definitions. We later show empirically that homogeneity for the unsimplified versions is similar.

% We consider the post-norm (Wang et al., 2019) transformer architecture underlying BERT and RoBERTa.
% HuggingFace BERT: 
% Pre-norm vs. post-norm: https://arxiv.org/pdf/1906.01787.pdf

\subsection{Transformer Definition}

The transformer encoder is a cascade of alternating \emph{multi-head self-attention} sublayers and \emph{feedforward} sublayers. Each multi-head self-attention sublayer can be further broken down as an aggregation of \emph{self-attention} heads.
Let $\lnorm(\cdot)$ denote a layer norm followed by a learned affine transformation.
Here we will consider the pre-norm transformer variant \citep{xiong2020layer}, meaning that $\lnorm$ comes before the residual connection wherever it appears.\footnote{The post-norm transformer applies these operations in the opposite order.} We will also assume that there are no biases, making all affine transformations into strict linear transformations.

\begin{definition}[Self-attention head]
Given parameters $W^k$, $W^q$, $W^v$ and input $X\in \mathbb{R}^{Tn}$, we define a \textit{self-attention head} $\selfattn$ as
\begin{align*}
    K &= W^k X \\
    Q &= W^q X \\
    V &= W^v X \\
    A &= \softmax(QK^\top / \sqrt{d_k}) \\
    H &= AV ,
\end{align*}
where $H$ is the output tensor.
\end{definition}

The multi-head self-attention sublayer computes several attention heads in parallel and aggregates them into a single sequence of vectors.

\begin{definition}[Multi-head self-attention sublayer]
Let $X \in \mathbb{R}^{Tn}$ be the input. We now define the \textit{$k$-multi-head self-attention sublayer} $\mselfattn_k$. First, we compute $k$ self-attention heads in parallel to produce $H_1, \cdots, H_k$. We then concatenate these along the feature axis to form $H$, and compute the sublayer output $Y$ as
\begin{equation*}
    \mselfattn_k(X) = \lnorm(WH) + X .
\end{equation*}
\end{definition}
Finally, the linear sublayer is the other component of the transformer.
\begin{definition}[Feedforward sublayer]
Let $X \in \mathbb{R}^{Tn}$ be the input. We compute the \textit{feedforward sublayer} $\ff$ according to
\begin{equation*}
    \ff(X) = \lnorm(W^f\relu(W^i X)) + X .
\end{equation*}
\end{definition}

\subsection{Results}

\begin{theorem}
If $X$ is ${\sim}1$-homogeneous in parameters $\theta$, then 
% there exists $\rho$ such that for all $\theta$ with $\rho < \norm{\theta}$,
$\selfattn(X; W^k, W^q, W^v)$ is ${\sim}1$-homogeneous in the concatenation of $\theta, W^k, W^q, W^v$.
\end{theorem}
\begin{proof}
Consider a self-attention layer receiving a ${\sim}1$-homogeneous input matrix $X \in \mathbb{R}^{Tn}$ where $T$ is the sequence length. Using the homogeneity rule for multiplication, $K,Q,V$ are each ${\sim}2$-homogeneous, as homogeneity is additive over multiplication. By the same argument, $QK^\top$ is ${\sim}4$-homogeneous.
In \autoref{thm:approx-softmax}, we show that if the input to softmax is approximately homogeneous, then the output is approximately scale-invariant.
Thus, $A$ is approximately $0$-homogeneous. Then $AV$ is ${\sim}1$-homogeneous.
\end{proof}
We show that the multi-head component that aggregates multiple heads into a shared representation also preserves approximate $1$-homogeneity.
\begin{theorem}
If $X$ is ${\sim}1$-homogeneous in parameters $\theta$, then
% there exists $\rho$ such that for all $\theta$ with $\rho < \norm{\theta}$,
$\mselfattn$ is ${\sim}1$-homogeneous in the full parameters.
\end{theorem}
\begin{proof}
Since $Wh$ is ${\sim}2$-homogeneous, $\lnorm(WH)$ is ${\sim}1$-homogeneous. The input $X$ is also ${\sim}1$-homogeneous by assumption, meaning that the sum is also ${\sim}1$-homogeneous.
\end{proof}
Finally, we turn to analyzing the feedforward sublayer of the transformer.
\begin{theorem}
If $X$ is ${\sim}1$-homogeneous, then $\ff(X; W^f, W^i)$ is ${\sim}1$-homogeneous in the full parameters.
\end{theorem}
\begin{proof}
Multiplying by each $W$ increases approximate homogeneity by $1$, and $\relu$ preserves approximate homogeneity.
So the input to $\lnorm$ is ${\sim}3$-homogeneous. Thus, its output is ${\sim}1$-homogeneous, and adding $X$ preserves approximate $1$-homogeneity.
\end{proof}
Together, these results suggest that the pre-norm transformer output is ${\sim}1$-homogeneous, assuming its input is ${\sim}1$-homogeneous. This precondition for the input holds in the ``base case'' of standard embeddings. By induction, we can imagine the output of a biasless pre-norm transformer encoder of any depth to be ${\sim}1$-homogeneous.

Interestingly, the homogeneity arguments do not work out if we instead consider the post-norm transformer architecture \citep{xiong2020layer}.

% \subsection{Other Transformer Variants}

% We have theoretically investigated approximate homogeneity for pre-norm transformers with biases removed.
% In \autoref{fig:trans-scaling-error}, we empirically investigate how the homogeneity is effected by changing the normalization or adding biases. \autoref{fig:trans-scaling-error} plots the cosine similarity of the network with parameters scaled by a constant $c$ to the encodings of the saturated network. Recall that a homogeneous network will produce a cosine similarity near $1$. As $c$ increases, we see that each transformer variant approaches a cosine similarity of $1$. The post-norm curve is more gradual than the pre-norm curve. Adding biases does not to significantly alter the rate of convergence.

% \begin{figure}
%     \centering
%     \includegraphics[width=.48\textwidth]{images/scales.pdf}
%     \caption{An empirical measure of the homogeneity of transformer variants. The $x$ axis represents a coefficient by which the parameters in the transformer are scaled, and the $y$ axis represents the cosine similarity to the representations in a saturated transformer.}
%     \label{fig:trans-scaling-error}
% \end{figure}

%% file: sections/integral_appendix.tex
\section{Sum Derivation} \label{sec:integral}
\subsection{Aligned Case} \label{sec:int-aligned}
Assume that $\norm{\theta_t} \approx 0$. Then,
\begin{align*}
    \norm{\theta_t} &\appropto \int \eta_t \norm{\theta_t} \mathrm{d} t \\
    \frac{\mathrm{d}}{\mathrm{d}t} \norm{\theta_t} &\appropto \eta_t \norm{\theta_t} \\
    \frac{\mathrm{d} \norm{\theta_t}}{\norm{\theta_t}} &\appropto \eta_t \mathrm{d} t \\
    \log \norm{\theta_t} &\appropto \int \eta_t \mathrm{d} t \\
    \norm{\theta_t} &\appropto \exp \left( \int \eta_t \mathrm{d} t \right) .
\end{align*}
Plugging in $\eta_t = 1/\sqrt{t}$, we get $\norm{\theta_t} \appropto \exp(\sqrt{t})$.
\subsection{Misaligned Case} \label{sec:int-misaligned}
First, we derive the sum approximation for $\norm{\theta_t}$.
We start with the fact that $\theta_{t+1} = \theta_t + \delta_t$ and misalignment, i.e., $\theta_t^\top \cdot \delta_t = 0$.
\begin{align*}
    \norm{\theta_{t+1}}^2
        &= (\theta_t + \delta_t) \cdot (\theta_t + \delta_t) \\
        &= \norm{\theta_t}^2 + \theta_t^\top \delta_t + \norm{\delta_t}^2 \\
        &= \norm{\theta_t}^2 + \norm{\delta_t}^2 \\
        &= \norm{\theta_0}^2 + \sum_{i=0}^t \norm{\delta_i}^2 .
\end{align*}
Taking the square root of both sides, $\norm{\theta_t}$ is roughly proportional to $\sum_{i=0}^t \norm{\delta_i}^2$.

Next, we show how to solve the integral, similarly to \autoref{sec:int-aligned}.
\begin{align*}
    \norm{\theta_t}^2 &\appropto \int \norm{\delta_t}^2 \mathrm{d} t\\
    \norm{\theta_t}^2 &\appropto \int \eta_t^2 \norm{\theta_t}^2 \mathrm{d} t \\
    \frac{\mathrm{d}}{\mathrm{d}t} \norm{\theta_t}^2 &\appropto \eta_t^2 \norm{\theta_t}^2 \\
    \frac{\mathrm{d} \norm{\theta_t}^2 }{\norm{\theta_t}^2} &\appropto \eta_t^2 \mathrm{d} t \\
    \log \norm{\theta_t}^2 &\appropto \int \eta_t^2 \mathrm{d} t .
\end{align*}
Now, we plug in the $\eta_t = 1 / \sqrt{t}$ learning rate:
\begin{align*}
    \log \norm{\theta_t}^2
        &\appropto \int ( 1/\sqrt{t})^2 \mathrm{d} t \\
        &\appropto \int \frac{\mathrm{d} t}{t} \\
        &\appropto \log t .
\end{align*}
So, in conclusion: $\norm{\theta_t} \appropto \sqrt{t}$.

%% file: sections/weight_decay.tex
\section{Weight Decay} \label{sec:weight-decay}

%\subsection{Weight Decay}
Weight decay regularizes the loss by the squared $\ell_2$ norm, modulated by a decay factor $\lambda$. For GD, this can be written
\begin{equation}
    \delta_t = -\eta_t \nabla_{\theta_t} L -  \lambda \theta_t . \label{eq:wd}
\end{equation}
Intuitively, the new term ${-}\lambda \theta_t$ will influence each step to point towards $0$.
Thus, large values of $\lambda$ might intuitively be expected to hinder or prevent norm growth.
% Whereas we expect large learning rates to be associated with norm growth \will{Discussion of learning rate no longer in paper}, large weight decay should hinder it. 
While we leave developing a more complete theoretical story to future work, 
here we empirically investigate the interplay of a constant learning rate $\eta$ and weight decay $\lambda$ by training a variety of transformer language models on Wikitext-2 for $1$ epoch.\footnote{$1$ epoch is chosen because of the computational cost of running this experiment over a large grid. In \autoref{sec:norm-growth}, we found that growth continued beyond $1$ epoch using the default AdamW settings.} We use the AdamW \citep{DBLP:journals/corr/abs-1711-05101} optimizer, varying $\lambda$ and $\eta$ across a range of common values, keeping all other hyperparameters constant.
\autoref{fig:norm-grids} visualizes the phase transition for norm growth as a function of $\lambda, \eta$.
% For SGD, we observe positive norm growth for $\eta$ roughly less than $1$e${-}3$ and $\lambda$ roughly less than $1$e${-}5$. For AdamW, 
The norm growth behavior seems to largely depend on weight decay, with a threshold for $\lambda$ lying between $0.01$ and $0.001$. While the trend likely depends on the optimizer, we can infer for AdamW at least that norm growth is probable when $\lambda = 0.01$, which is a common choice, e.g., reflecting default settings in PyTorch.
% As the network architecture is fixed across all runs, the effect of hyperparameters like the depth and width is unclear.
Thus, while large values of $\lambda$ will indeed hinder norm growth, we find preliminary empirical evidence that standard choices (${\sim}0.01$) do not prevent it.

\begin{figure}[t]
    \centering
    \includegraphics[width=\columnwidth]{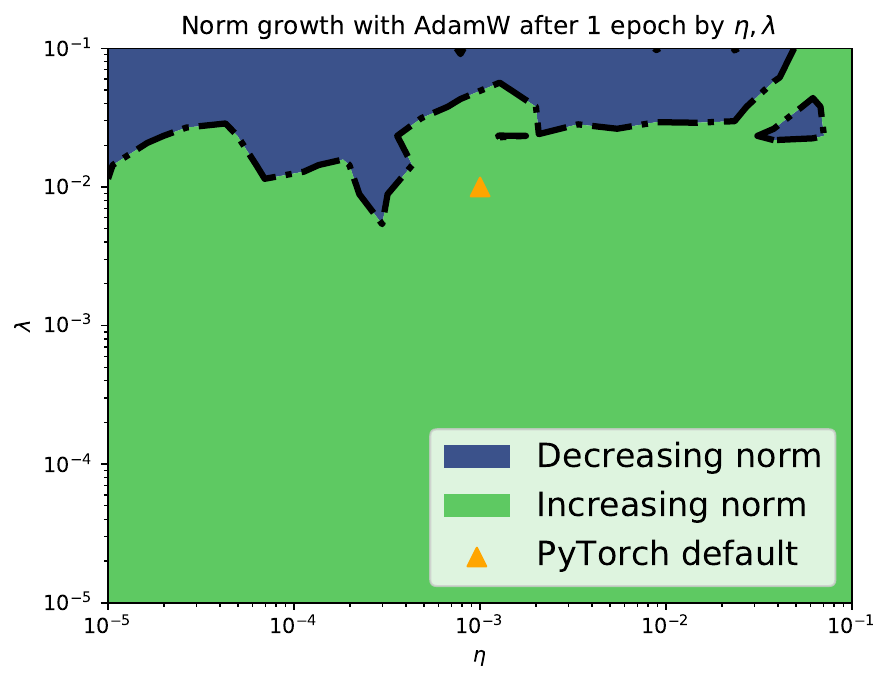}
    \caption{Norm growth over the first epoch, varying $\eta, \lambda$. The triangle shows the default AdamW hyperparameters in PyTorch.}
    \label{fig:norm-grids}
\end{figure}